%% file: main.tex
\documentclass[11pt]{article}

\usepackage[utf8]{inputenc}
\usepackage[T1]{fontenc}
\usepackage{lmodern}
\usepackage{amsmath,amssymb,amsthm,mathtools}
\usepackage{booktabs,tabularx,array,multirow}
\usepackage{graphicx}
\usepackage[margin=1in]{geometry}
\usepackage[round,authoryear]{natbib}
\usepackage{microtype}
\usepackage[hidelinks]{hyperref}
\usepackage{enumitem}
\usepackage{xcolor}
\usepackage{caption}
\usepackage{subcaption}
\usepackage{authblk}

\graphicspath{{../figures/}}
\newcommand{\alttext}[1]{%
  \par\smallskip
  \begingroup
  \footnotesize
  \raggedright
  \textit{Alt text:} #1\par
  \endgroup
}
\newcolumntype{Y}{>{\raggedright\arraybackslash}X}

\newtheorem{theorem}{Theorem}
\newtheorem{proposition}{Proposition}

\newtheorem{assumption}{Assumption}

\newcommand{\E}{\mathbb{E}}
\newcommand{\Pp}{\mathbb{P}}
\newcommand{\Pn}{\mathbb{P}_n}
\newcommand{\Ccal}{\mathcal{C}}

\newcommand{\psihat}{\widehat\psi}
\newcommand{\etahat}{\widehat\eta}
\DeclareMathOperator{\expit}{expit}

\title{Data-Poisoning Audits for Causal Effect Estimation}

\author[1]{Kwangho Kim}
\affil[1]{%
Department of Statistics, Korea University, Seoul, Republic of Korea\\
\href{mailto:kwanghk@korea.ac.kr}{\texttt{kwanghk@korea.ac.kr}}%
}
\date{}

\begin{document}
\maketitle

\begin{abstract}
Observational causal analyses increasingly pool records across sites,
vendors, and collection systems, creating vulnerability to append-only
attacks in which plausible records are strategically selected to alter a
reported treatment effect. We develop a data-poisoning audit for augmented
inverse-probability-weighted estimation. The analyst specifies a finite
catalog of feasible records, an append budget, and nested source
capacities, and the adversary selects a feasible subset to maximize
movement in a prespecified direction. With preprocessing and nuisance fits held fixed, we propose a greedy scan
that computes the exact finite-sample worst-case movement at every append
budget. To account for nuisance refitting, we go on to derive a
total-influence score combining each record's direct contribution with its
effect through the propensity and outcome models. We further obtain a conservative
finite-budget bound for the fully refitted estimate. Extensive simulations validate the exact result and show
that total influence improves local refit prediction, while multisite and
public-data analyses demonstrate material sensitivity at small append
budgets. By translating adversarial data-composition risk into movement curves and
critical budgets, the framework supports more reliable causal reporting
and the design of source-level safeguards.
\end{abstract}

\noindent\textbf{Keywords:} causal inference; data poisoning; average treatment effect; adversarial robustness; influence functions.

\section{Introduction}\label{sec:introduction}

Causal effect estimates from observational studies increasingly inform
drug safety, public policy, product design, and automated decision systems. The
underlying data are often pooled across hospitals, registries, contractors,
vendors, or brokers. In such settings, records may be contributed through
sources with heterogeneous governance and quality controls, and some
sources may be able to add strategically selected observations that pass
routine validation. Reliability therefore depends not only on
identification and model specification, but also on how far the composition
of the observed data can be altered under realistic contribution limits.

We use the term \emph{data-poisoning audit} to mean a prospective
assessment of this vulnerability, not a forensic test of whether poisoning
has already occurred. Given a fitted causal-analysis pipeline, a finite
catalog of plausible additional records, and source-level contribution
limits, the audit reports the largest upward or downward change attainable
at each append budget, the minimum budget required to cross a prespecified
decision threshold, and the extent to which nuisance-model refitting alters
that vulnerability. These quantities translate an abstract security concern
into an operational measure of how easily a scientific or policy conclusion
can be changed.

Yet familiar robustness guarantees in causal inference do not address
such strategic changes in sample composition. Double robustness is a consistency
property of estimators such as augmented inverse-probability weighting
(AIPW): for the average treatment effect, consistency is retained if either
the propensity model or the outcome regression is correctly specified
\citep{robins1994,bang2005,funk2011}. Neyman orthogonality, which underlies one-step, AIPW, and
double-machine-learning procedures, makes the estimating equation locally
insensitive to first-order perturbations of the nuisance functions
\citep{kennedy2024semiparametric}, while targeted procedures update the
initial nuisance fits so that the empirical efficient-influence-function
equation is approximately solved
\citep{vanderlaan2006,tsiatis2006}. These properties protect against model
or estimation error, not against strategically selected records.

Strategic vulnerability enters through the same score structure that
supports efficient estimation: inverse-propensity-weighted outcome
residuals can become large in regions of limited overlap. Thus, although
AIPW is protected against certain nuisance-model errors, it need not be
stable to adversarially selected records. This mechanism motivates the
proposed audit. Here, we do not claim novelty for the underlying instability
itself: robust statistics has long distinguished model robustness from
gross-error sensitivity \citep{huber1964,hampel1986}, and causal weighting
methods already address unstable weight tails
\citep{crump2009,imai2014,li2018,vanderlaan2025,hillchaudhuri2025}.

This distinction also clarifies how the proposed audit relates to conventional causal sensitivity analysis. Analyses for unmeasured confounding typically hold the observed sample fixed while relaxing the assumptions used to identify the causal effect \citep[e.g.,][]{rosenbaum2002,yadlowsky2022}. Our audit instead holds the identification model and adjustment set fixed while allowing feasible records to be appended subject to explicit source-capacity constraints. Its outputs quantify movement of the reported estimator rather than identification bounds for the population effect, and therefore do not replace conventional sensitivity analyses. A result may be insensitive to plausible unmeasured confounding yet vulnerable to a small number of admissible record additions, or resistant to such additions while remaining highly sensitive to violations of exchangeability. Consequently, both diagnostics may be necessary in high-stakes observational studies.

A useful audit must represent the data-contribution mechanism without making the optimization intractable. Unlimited repetition of one worst record is often implausible because duplicate checks, account restrictions, or source-level quotas limit repeated contributions. At the other extreme, an unrestricted continuous bilevel attack requires a difficult model of record plausibility and repeated nuisance refitting. We take an intermediate approach. The analyst specifies a finite catalog of records that would pass the relevant scientific and operational checks, together with nested capacity constraints: candidate alternatives may belong to a common profile, profiles to contributors, contributors to sites, and sites to a network. The resulting conclusions are conditional on this declared threat model, just as the conclusions of an unmeasured-confounding analysis depend on its chosen sensitivity model.

These constraints yield an exact and interpretable audit when
preprocessing and nuisance fits remain fixed at their reference-sample
values. In this setting, AIPW is an average of record-level scores, so
each candidate has a known directional gain relative to the reference
estimate. When the profile, contributor, site, and network groups are
nested, sorting candidates by gain and accepting each one that preserves
all capacities yields the exact worst-case append set at every feasible
cardinality. A single accepted ordering therefore determines the full
movement curve and the minimum append budget required to cross a decision
threshold. Disjoint source partitions are a special case.

In practice, deployed pipelines commonly refit their propensity and outcome
models after records are appended, so exact fixed-pipeline movement must be
distinguished from realized movement after refitting. For
finite-dimensional penalized nuisance models, we derive a per-record total
influence that combines the direct AIPW-score contribution with the
indirect effect of nuisance refitting. It supports candidate ranking after
a single linear-system solve, rather than one complete refit per candidate.
For capped inverse weights and stable ridge nuisance fits, we further
derive a conservative finite-budget bound for the fully refitted estimate.
Total influence provides a practical first-order diagnostic, whereas the
bound gives a finite-budget guarantee under the stated conditions;
separating them avoids presenting a local approximation as an exact refit
result.

Figure~\ref{fig:running-example} illustrates the proposed audit using a
simple two-site data-generating process, where \(S\in\{1,2\}\) indexes the
source site and \(A\in\{0,1\}\) is the binary treatment:
\[
\begin{aligned}
&\Pr(S=s')=\tfrac12,\quad
X\mid S=s'\sim N(1.2s'-1.8,1),
\quad s'\in\{1,2\},\\
&A\mid X\sim\operatorname{Bernoulli}\{\operatorname{expit}(2.5X)\},
\quad
Y=0.5A+X+\varepsilon,\quad \varepsilon\sim N(0,1).
\end{aligned}
\]
The site indicator records only the source of each row, while \(X\) is the
sole adjustment covariate, and the true treatment effect is \(0.5\).
Because treatment is uncommon at small \(X\), a treated record in that
region can receive a large inverse-propensity weight and, if its outcome is
low, pull the estimate downward. A rare untreated record with large \(X\)
and a high outcome can have a similar effect. We illustrate a downward
audit because the original estimate is positive; the same procedure
applies in the opposite direction.

Panel~A highlights the two candidate patterns with the greatest downward
impact. Panel~B ranks candidates by their predicted downward effect. Multiple alternatives may correspond to
the same profile, but at most one may be selected, and site and overall
capacities impose further limits; hatched candidates are skipped because
they would violate a capacity. Panel~C appends the selected records
cumulatively and compares the exact fixed-pipeline change, a local
approximation under nuisance refitting, and a complete reanalysis. The
first two become negative after two additions, whereas the complete
reanalysis remains positive until the fourth, illustrating the distinction
between fixed-pipeline movement and realized movement after refitting.

\begin{figure}[t]
\centering
\includegraphics[width=\textwidth]{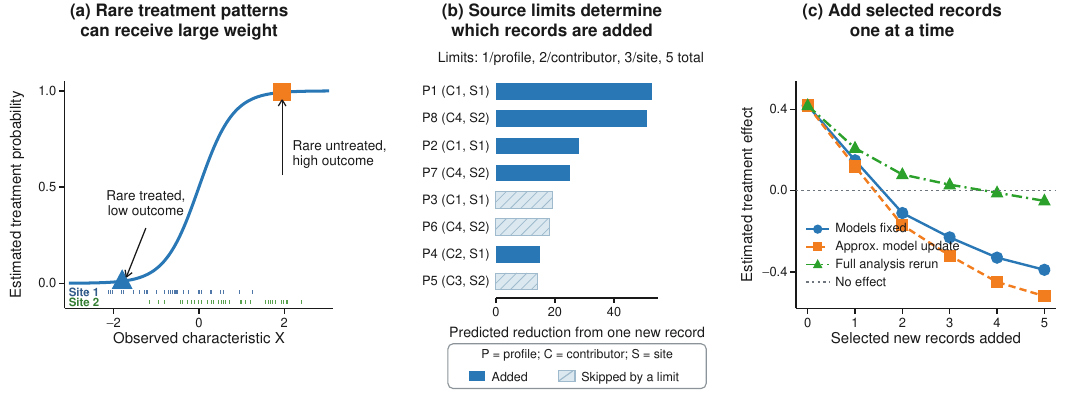}
\caption{
Running example of a downward data-poisoning audit.
\textbf{A.} Rare treatment patterns can give plausible new records
large inverse-propensity weight.
\textbf{B.} Candidates are ranked by predicted reduction under the fixed
pipeline. Capacities specify the maximum number selectable from each
group: one per profile, two per contributor, three per site, and five
overall. Hatched bars mark candidates skipped because a capacity has
already been reached.
\textbf{C.} The selected records are added cumulatively, comparing the
exact fixed-model calculation, a local model-update approximation, and
a complete reanalysis.
}
\alttext{Three-panel illustration of a downward audit. Panel A shows rare
treated and untreated profiles with large inverse weights. Panel B ranks
candidate records and marks those accepted or skipped under nested source
capacities. Panel C compares fixed-pipeline, first-order refit, and fully
refitted treatment-effect estimates as selected records are appended.}
\label{fig:running-example}
\end{figure}

The running example illustrates the three quantities developed in the
paper: exact movement under a fixed pipeline, a first-order approximation
under nuisance refitting, and realized movement after complete reanalysis.
Building on these quantities, the paper makes four contributions. First,
we formulate a hierarchical source-capacity threat model and derive the exact fixed-pipeline AIPW movement curve at every feasible budget. Second, we develop a total-influence diagnostic that
accounts for nuisance refitting, together with a conservative finite-budget
bound for capped inverse weights and stable ridge nuisance fits. Third, we
provide a decision-based rule for choosing the inverse-weight cap,
re-ranking candidates for each cap rather than transferring an attack
optimized for a different estimator. Fourth, we evaluate the audit across
overlap and nuisance-misspecification regimes, under multilevel contributor
and site capacities, on ten Atlantic Causal Inference Conference 2016
instances, and in a smoking-cessation analysis.

The audit is intended to accompany, rather than replace, ordinary causal
diagnostics. Its interpretation requires the analyst to report the
estimand, fitted pipeline, feasible record catalog, source capacities,
and whether the movement is computed for a fixed pipeline, approximated
under refitting, or obtained from a complete reanalysis. A poisoning budget
has no stable meaning without these specifications.

\subsection{Relationship to existing work}

Our contribution intersects robust causal estimation, causal sensitivity analysis, and training-data security. Outlier-resistant causal methods use density-power weighting \citep{harada2024}, bounded-influence estimating equations and balancing propensity scores \citep{lee2025}, high-dimensional covariate balance \citep{ning2020}, or inverse-probability-tail trimming \citep{hillchaudhuri2025}. Other approaches perturb target populations or nuisance laws \citep{jeong2020,tanimoto2025}, modify measurement-error models \citep{kallus2018,guo2022}, or bound failures of propensity-score and positivity assumptions \citep{csillag2024}. These methods address random contamination, distributional uncertainty, or violations of causal assumptions; they do not optimize directionally chosen record additions against a fitted treatment-effect pipeline. Our audit conditions on an adjustment set and evaluates uncertainty in data composition, so an applied analysis may require both ordinary identification sensitivity analyses and the proposed poisoning audit.

Methods for stabilizing overlap are particularly relevant because they control the same inverse weights that create the attack surface studied here. Trimming changes the target population \citep{crump2009}; overlap weights emphasize units near treatment equipoise \citep{li2018}; balancing and isotonic calibration stabilize fitted weights \citep{imai2014,vanderlaan2025}; and \citet{desai2019} connect estimation-error bounds with clipping. Our target is different: the largest directional movement achievable through feasible, capacity-constrained record additions. We therefore evaluate inverse-weight capping jointly in terms of clean-data performance and worst-case audited movement, without claiming that capping provides universal protection.

Data-poisoning attacks are commonly formulated as constrained or bilevel
optimization problems \citep{biggio2012,mei2015,jagielski2018}. Influence
functions provide tractable local approximations to training-point effects
\citep{koh2017}, while certified defenses bound changes induced by training
data perturbations \citep{steinhardt2017}. Most closely, Abstract Gradient
Training (AGT) propagates sound bounds through first-order training procedures
to enclose parameters reachable after additions, deletions, or bounded
perturbations \citep{sosnin2026}. Our audit is narrower but exact for a fixed scalar AIPW estimator under a finite feasible-record catalog and hierarchical source capacities; its
treatment of nuisance refitting is local rather than a global
parameter-space certificate.

We do not claim the first adversarial manipulation of a causal estimand.
Prior work poisons off-policy policy-value estimators \citep{lobo2022},
optimizes treatment assignments in randomized trials \citep{babaei2024},
attacks causal procedures through missingness
\citep{koyuncu2023,koyuncu2026}, and steers Bayesian posteriors through
deletion or replication \citep{carreau2025}. To our knowledge, prior work
has not combined a finite append-only record catalog and nested source
capacities with exact fixed-pipeline AIPW movement curves, refit-aware
influence ranking, and a finite-budget bound for the fully refitted capped
estimator.

\begin{table}[t]
\centering
\small
\caption{Positioning relative to adjacent methodological areas.}
\label{tab:positioning}
\begin{tabularx}{\textwidth}{
    p{0.22\textwidth}
    p{0.22\textwidth}
    p{0.23\textwidth}
    Y
}
\toprule
Research line
& Perturbation or concern
& Primary target
& Difference in focus \\
\midrule

Robust causal estimation
\citep{harada2024,lee2025,hillchaudhuri2025}
& Outliers, contamination, or heavy-tailed weighted outcomes
& Robust estimation and inference
& Strategically selected append sets, directional estimator movement,
and explicit source capacities \\

Overlap and weight stabilization
\citep{crump2009,li2018,vanderlaan2025}
& Limited overlap and extreme inverse weights
& Variance reduction, covariate balance, or alternative target populations
& Audits the adversarial exposure created by extreme weights and compares
caps under cap-specific re-optimized attacks \\

Causal-adversarial methods
\citep{lobo2022,babaei2024,koyuncu2026}
& Policy-data perturbations, treatment assignments, or missingness patterns
& Policy-value or treatment-effect error
& Append-only observational AIPW with a feasible-record catalog,
nested source capacities, and nuisance refitting \\

General poisoning certificates and AGT
\citep{steinhardt2017,sosnin2026}
& Training-data additions, deletions, replacements, or bounded perturbations
& Predictive loss or reachable model parameters
& Exact finite-sample movement for a fixed scalar AIPW pipeline, with a
first-order refit diagnostic and a conservative finite-budget refit bound \\
\bottomrule
\end{tabularx}
\end{table}

\section{Methods}\label{sec:methods}

\subsection{Audited causal pipeline}\label{sec:pipeline}

Let \(O=(X,A,Y)\), where \(A\in\{0,1\}\) is a binary treatment,
\(X\) contains pretreatment covariates, and \(Y\) is an outcome. The
reference sample is \(\mathcal D_n=\{O_i\}_{i=1}^n\). Under consistency,
conditional exchangeability, and positivity
\citep[][Chapter~12]{imbens2015causal}, the average treatment effect (ATE)
\[
\psi=\E\{Y(1)-Y(0)\}
\]
is identified as
\[
\psi=\E\{\mu_1(X)-\mu_0(X)\},
\]
where
\[
e(x)=\Pp(A=1\mid X=x),
\quad
\mu_a(x)=\E(Y\mid A=a,X=x).
\]
The proposed audit does not weaken these identification assumptions. It
evaluates the sensitivity of a specified estimator after the target
population and adjustment set have been chosen.

For a nuisance collection \(\eta=(e,\mu_0,\mu_1)\), define the
uncentered AIPW score
\begin{equation}\label{eq:aipw-score}
\varphi(O;\eta)
=
\mu_1(X)-\mu_0(X)
+\frac{A}{e(X)}\{Y-\mu_1(X)\}
-\frac{1-A}{1-e(X)}\{Y-\mu_0(X)\}.
\end{equation}
At the population nuisance functions, the centered AIPW score
\[
\phi(O;\eta)
=
\varphi(O;\eta)-\psi
\]
is the efficient influence function (EIF) for the ATE under the
nonparametric model. We take as the reference estimate the AIPW estimator
\[
\psihat
=
\Pn\{\varphi(O;\etahat)\},
\]
computed from the original sample \(\mathcal D_n\). Its empirically centered fitted score is
\[
\widehat\phi(O)
=
\varphi(O;\widehat\eta)-\widehat\psi.
\]

For implementation and refit analysis, we estimate the nuisance functions
with finite-dimensional working models: ridge logistic regression for
treatment and arm-specific ridge linear regressions for the outcome. These
models restrict nuisance estimation, not the nonparametric model defining
the estimand and EIF, and make the refit map transparent.

Let
\[
\theta=(\beta^\top,\gamma_0^\top,\gamma_1^\top)^\top
\]
collect the coefficients of the propensity and outcome working models. For notational convenience, we write
\[
\eta_\theta
=
(e_\beta,\mu_{0,\gamma_0},\mu_{1,\gamma_1}),
\quad
\varphi_\theta(O)
=
\varphi(O;\eta_\theta).
\]
At the fitted coefficients \(\widehat\theta\), we correspondingly write
\[
\widehat\eta=\eta_{\widehat\theta},
\quad
\varphi_{\widehat\theta}(O)
=
\varphi(O;\widehat\eta).
\]

The primary audit is developed for the ordinary AIPW estimator. We later
consider inverse-weight capping as a tunable intervention that enables a
finite refit safety bound and trades clean-data precision against audited
poisoning exposure.

The preprocessing map, outcome range, nuisance-model penalties, and
catalog rules are fixed using the reference sample. When inverse-weight
capping is used, its threshold is fixed as well. Appended records do not
alter standardization or validation ranges. This represents a versioned
deployment and prevents the audit from conflating poisoning of
preprocessing with poisoning of nuisance estimation.

\subsection{Feasible catalog and nested source capacities}
\label{sec:threat}

A white-box adversary knows the fitted analysis pipeline and may append at
most \(m\) records selected from a finite catalog
\[
\Ccal=\{z_j=(x_j,a_j,y_j):j=1,\ldots,J\},
\]
whose entries satisfy the relevant scientific and operational
plausibility requirements.

The candidates are organized into nested source groups: alternative
records for one profile form a profile group, profiles are nested within
contributors, contributors within sites, and sites within a network. Let
\(\mathcal L\) denote this collection. Any two groups in \(\mathcal L\)
are either disjoint or one is contained in the other. Each
\(L\in\mathcal L\) has a capacity \(c_L\), the maximum number of selected
records allowed from that group. For example, a profile capacity of one
permits only one alternative, while higher-level capacities limit
concentration within contributors, sites, or the network.

An index set \(I\subseteq\{1,\ldots,J\}\) is feasible if
\begin{equation}\label{eq:capacity}
|I|=k\le m,
\quad
\big|\{j\in I:z_j\in L\}\big|\le c_L
\quad\text{for every }L\in\mathcal L.
\end{equation}
Throughout, \(m\) is taken not to exceed the maximum cardinality allowed
by the capacity constraints; any larger declared budget is equivalent to
that maximum.

Thus, the same candidate may be subject simultaneously to profile-,
contributor-, site-, and network-level limits. Disjoint source groups form
a simple partition. By contrast, constraints that overlap without being
nested, such as those involving contributors operating across multiple
sites and time windows, require a more general optimization method; the
greedy result below does not apply.

The catalog should reflect the application. It may contain observed
covariate profiles, bounded perturbations of observed records, or records
generated by a validated model. A profile capacity can prevent several
mutually exclusive versions of the same individual from being selected,
while contributor and site capacities limit concentration within a single
source.

For a prespecified direction \(s\in\{-1,+1\}\), let \(s=-1\) denote a
downward audit and \(s=+1\) an upward audit. Conditional on \(s\), the
audit optimizes over feasible appended records. The fixed-pipeline gain of
candidate \(j\) is
\begin{equation}\label{eq:gain}
d_j(s)
=
s\,\widehat\phi(z_j)
=
s\{\varphi(z_j;\etahat)-\psihat\}.
\end{equation}
Unless one direction is dictated by the decision problem, the audit is
run separately for both values of \(s\). A larger \(d_j(s)\) indicates
greater movement in direction \(s\) with the fitted nuisance functions
held fixed. Candidates are scanned in decreasing order of \(d_j(s)\) and
accepted whenever doing so preserves all capacities. Let
\[
\pi_{s,1},\pi_{s,2},\ldots
\]
denote the accepted order, as illustrated in panel~B of
Figure~\ref{fig:running-example}.

\paragraph{Specification and interpretation.}
The catalog \(\Ccal\), source hierarchy \(\mathcal L\), capacities
\(\{c_L:L\in\mathcal L\}\), and maximum append budget \(m\) should be
fixed before inspecting the candidate rankings. They should encode the
records and contribution patterns that are scientifically and
operationally plausible in the application. Useful specifications include
an observed-profile catalog, a local-plausibility catalog allowing
bounded perturbations, and a source-governance hierarchy imposing
contributor and site limits. Because the audit is conditional on these
choices, results should be compared across plausible catalog and
capacity specifications.

The audit measures prospective vulnerability rather than detecting whether
poisoning has occurred. Large movement under a small append budget
indicates that the reported estimate is highly sensitive to plausible
changes in data composition, whereas small movement over a broad range of
budgets indicates greater stability under the specified threat model.
Neither finding establishes whether the data were actually manipulated.
Any assurance is conditional on the stated catalog, capacities, and
append-only setting; omitted record classes, changes to preprocessing, and
replacement attacks remain outside the guarantee.

\subsection{Exact audit under a fixed pipeline}
\label{sec:fixed}

We call the pipeline \emph{fixed} when its preprocessing and fitted
nuisance functions remain at their reference-sample values. For a
feasible set \(I\) of \(k\) appended records, the resulting estimate is
\[
\psihat_{\mathrm{fix}}^{\,I}
=
\frac{
n\psihat+\sum_{j\in I}\varphi(z_j;\etahat)
}{n+k},
\]
and its movement in direction \(s\in\{-1,+1\}\) is
\begin{equation}\label{eq:fixed-movement}
s\{\psihat_{\mathrm{fix}}^{\,I}-\psihat\}
=
\frac{\sum_{j\in I}d_j(s)}{n+k}.
\end{equation}
For fixed \(s\) and \(k\), the adversary therefore selects the feasible
\(k\)-record set with the largest total gain
\(\sum_{j\in I}d_j(s)\). The following theorem shows that, under nested
source capacities, this optimum is obtained by scanning candidates in
decreasing order of \(d_j(s)\) and accepting each candidate that preserves
feasibility.

\begin{theorem}
\label{thm:exact}
Fix \(s\in\{-1,+1\}\). For each feasible \(k\), the first \(k\)
accepted candidates maximize \eqref{eq:fixed-movement} over all feasible
\(k\)-record append sets. The corresponding worst-case movement is
\begin{equation}\label{eq:exact-k}
B^{\mathrm{fix}}_s(k)
=
\frac{1}{n+k}
\sum_{r=1}^{k}d_{\pi_{s,r}}(s),
\quad
B^{\mathrm{fix}}_s(0)=0.
\end{equation}
With a budget of at most \(m\) records, the worst-case movement is
\[
A^{\mathrm{fix}}_s(m)
=
\max_{0\le k\le m}B^{\mathrm{fix}}_s(k),
\]
and the accepted prefix attaining this maximum is optimal.
\end{theorem}
Here, \(B^{\mathrm{fix}}_s(k)\) is the largest movement achievable with
exactly \(k\) appended records, whereas \(A^{\mathrm{fix}}_s(m)\) allows
the adversary to use any number up to \(m\). The maximum over \(k\) is
needed because appending another record changes both the total gain and
the normalizing denominator \(n+k\); consequently, the strongest attack
may use fewer than \(m\) records.

The nested source capacities give the feasible sets the greedy structure
required by Theorem~\ref{thm:exact}. If each candidate belongs to at most
\(H_{\mathcal L}\) nested groups, one sorting step and one feasibility scan
compute the full curve in
\(O(J\log J+JH_{\mathcal L})\) time. Upward and downward audits require
separate scans. A proof of Theorem~\ref{thm:exact} is given in
Appendix~\ref{sec:supp-proof-exact}, and numerical verification is
reported in Appendix~\ref{sec:supp-validation}.

At budget \(m\), the range of estimates attainable under the fixed
pipeline is
\begin{equation}\label{eq:sensitivity-band}
\left[
\psihat-A^{\mathrm{fix}}_{-1}(m),\,
\psihat+A^{\mathrm{fix}}_{+1}(m)
\right].
\end{equation}
This is a deterministic movement range conditional on the catalog and
capacities, not a confidence interval.

Let \(D\ge0\) denote the movement required to cross a prespecified
decision threshold. The corresponding decision-instability budget is
\begin{equation}\label{eq:critical}
m_s^\star(D)
=
\min\{m\ge0:A_s^{\mathrm{fix}}(m)\ge D\},
\end{equation}
with \(m_s^\star(D)=\infty\) if the threshold cannot be reached. For a
positive estimate, taking \(s=-1\) and \(D=\psihat\) gives the minimum
budget required to reverse its sign. Taking
\[
D
=
\max\left\{
0,\,
\psihat-z_{0.975}\widehat{\mathrm{se}}(\psihat)
\right\}
\]
gives the minimum downward budget required to move the lower endpoint of
the ordinary Wald interval to zero while holding its standard error
fixed. If the original lower endpoint is already nonpositive, then
\(D=0\) and the corresponding decision-instability budget is zero. The
standard error is estimated from the reference sample and held fixed, so
this calculation captures instability arising from movement of the point
estimate but not changes in uncertainty after records are appended. It is
therefore a decision-sensitivity calculation, not a coverage guarantee
under poisoning.

\subsection{Influence under nuisance refitting}
\label{sec:total-if}

The exact audit above keeps the preprocessing and fitted nuisance functions
fixed. In practice, however, the nuisance models would typically be
refitted after records are appended. Here, we derive a first-order
approximation to the change in the AIPW estimate that accounts for this
refitting.

Let \(U(O;\theta)\) denote the penalized estimating function, including
the ridge contribution, with \(\widehat\theta\) satisfying
\begin{equation}\label{eq:estimating-equation}
\Pn U(O;\widehat\theta)=0.
\end{equation}
Define
\[
H
=
-\Pn\partial_\theta U(O;\widehat\theta),
\quad
\bar q
=
\Pn\partial_\theta\varphi_{\widehat\theta}(O).
\]

For a candidate record \(z\), consider the smooth contamination path
\[
P_t=(1-t)\Pn+t\delta_z,
\]
where \(\delta_z\) is the Dirac measure at \(z\), and let \(\theta(t)\)
satisfy
\[
P_tU\{O;\theta(t)\}=0.
\]
Because \(\theta(0)=\widehat\theta\), implicit differentiation at \(t=0\)
gives, when \(H\) is nonsingular,
\[
\dot\theta(0)
=
H^{-1}U(z;\widehat\theta).
\]
Differentiating the corresponding refitted AIPW estimate
\(P_t\varphi_{\theta(t)}\) at \(t=0\) then gives the total influence of an appended record
\(z\).

\begin{proposition}
\label{prop:total-if}
Suppose \(U(O;\theta)\) and \(\varphi_\theta(O)\) are continuously
differentiable with respect to \(\theta\) in a neighborhood of
\(\widehat\theta\), and \(H\) is nonsingular. Then
\begin{equation}\label{eq:total-if}
\left.
\frac{\mathrm d}{\mathrm dt}
P_t\varphi_{\theta(t)}
\right|_{t=0}
=
\operatorname{IF}_{\mathrm{tot}}(z)
=
\underbrace{\widehat\phi(z)}_{\text{direct contribution}}
+
\underbrace{
\bar q^{\!\top}H^{-1}U(z;\widehat\theta)
}_{\text{nuisance-refit contribution}}.
\end{equation}
\end{proposition}

The first term is the direct fixed-pipeline contribution, whereas the
second captures the effect of nuisance-model refitting. In particular,
\[
P_t\varphi_{\theta(t)}-\widehat\psi
=
t\,\operatorname{IF}_{\mathrm{tot}}(z)+o(t)
\qquad
\text{as }t\to0.
\]
Because appending one copy of \(z\) corresponds to a contamination
fraction of \(1/(n+1)\), we use
\[
\frac{\operatorname{IF}_{\mathrm{tot}}(z)}{n+1}
\]
as the first-order prediction of the resulting refitted movement. Computationally, one solves
\[
H^\top v=\bar q
\]
once and then evaluates
\[
\operatorname{IF}_{\mathrm{tot}}(z)
=
\widehat\phi(z)+v^\top U(z;\widehat\theta)
\]
for each candidate. The derivation is given in
Appendix~\ref{sec:supp-proof-total-if}, and the corresponding expansion
for multiple appended records appears in
Appendix~\ref{sec:supp-multiple-records}.

For a feasible set \(I\) of \(k\) appended records, the first-order
directional movement is
\[
\frac{1}{n+k}
\sum_{j\in I}
s\,\operatorname{IF}_{\mathrm{tot}}(z_j).
\]
Accordingly, we define
\[
d_j^{\mathrm{tot}}(s)
=
s\,\operatorname{IF}_{\mathrm{tot}}(z_j).
\]
Because this first-order objective is additive, applying the same
capacity-constrained scan to \(d_j^{\mathrm{tot}}(s)\) gives the accepted
order
\[
\pi^{\mathrm{tot}}_{s,1},
\pi^{\mathrm{tot}}_{s,2},
\ldots.
\]
The corresponding local curves are
\begin{align}
B_s^{\mathrm{tot}}(k)
&=
\frac{1}{n+k}
\sum_{r=1}^{k}
d_{\pi^{\mathrm{tot}}_{s,r}}^{\mathrm{tot}}(s),
\quad
B_s^{\mathrm{tot}}(0)=0,
\label{eq:total-curve}\\
A_s^{\mathrm{tot}}(m)
&=
\max_{0\le k\le m}
B_s^{\mathrm{tot}}(k).
\label{eq:total-at-most}
\end{align}

These curves are first-order diagnostics under nuisance refitting, rather
than exact finite-budget guarantees. As the appended mass increases,
nuisance-model curvature, candidate rankings, and the optimal selected set
may change. We therefore use the curves to select candidate records and
then rerun the prescribed pipeline, retaining the fixed preprocessing
while refitting the nuisance models, to report the realized stress-test
movement.

The direct and total-influence rankings need not agree. A record with a
large direct contribution may be partly absorbed by nuisance refitting,
whereas a record with smaller direct leverage may strongly perturb an
ill-conditioned nuisance-model direction and thereby affect predictions
across many reference observations. The overlap between the selected sets
therefore helps determine whether vulnerability is driven primarily by
direct score leverage or by nuisance-model refitting.

\subsection{Finite-budget bound under nuisance refitting}
\label{sec:safety}

The fixed-pipeline audit in Section~\ref{sec:fixed} gives an exact
finite-budget result but does not account for nuisance-model refitting,
whereas the total-influence analysis accounts for refitting only through
a first-order approximation. A finite-budget bound for the fully refitted
estimate requires uniform control of the AIPW score as the nuisance
parameters change. Rather than impose uniform overlap under every feasible
refit, we cap the inverse-propensity weights and impose stability
conditions on the penalized nuisance fits.

For a fixed cap \(M>1\), define
\[
w_{1,M}(x)
=
\min\{e(x)^{-1},M\},
\quad
w_{0,M}(x)
=
\min\{[1-e(x)]^{-1},M\},
\]
and the capped AIPW score
\begin{equation}\label{eq:capped-score}
\begin{aligned}
\varphi_M(O;\eta)
={}&
\mu_1(X)-\mu_0(X)
+A\,w_{1,M}(X)\{Y-\mu_1(X)\}\\
&-(1-A)\,w_{0,M}(X)\{Y-\mu_0(X)\}.
\end{aligned}
\end{equation}
The corresponding reference estimate is
\[
\psihat_M
=
\Pn\{\varphi_M(O;\etahat)\}.
\]
Because capping modifies the AIPW score, we denote its centered fitted
version by
\[
\widehat r_M(z)
=
\varphi_M(z;\etahat)-\psihat_M,
\]
rather than by \(\widehat\phi_M(z)\). Capping changes the estimator and
may introduce bias when the outcome regression is misspecified. We
therefore treat \(M\) as a tunable intervention chosen from the reference
sample and held fixed throughout the audit.

For \(s\in\{-1,+1\}\), define the capped directional gain
\[
d_{j,M}(s)
=
s\,\widehat r_M(z_j).
\]
Applying the capacity-constrained scan from Section~\ref{sec:fixed} gives
the accepted order
\[
\pi_{s,M,1},\pi_{s,M,2},\ldots
\]
and the exact fixed-pipeline curve
\begin{equation}\label{eq:capped-fixed-curve}
B^{\mathrm{fix}}_{s,M}(k)
=
\frac{1}{n+k}
\sum_{r=1}^{k}
d_{\pi_{s,M,r},M}(s),
\quad
B^{\mathrm{fix}}_{s,M}(0)=0.
\end{equation}

To control the additional change caused by nuisance refitting, consider a
feasible set \(I\) of \(k\) appended records and define
\[
P_{n,I}
=
\frac{
n\Pn+\sum_{j\in I}\delta_{z_j}
}{n+k}.
\]
Write the penalized nuisance objective under a probability measure \(P\)
as
\[
F_P(\theta)
=
P\ell(O;\theta)
+
\frac{1}{2}\theta^\top\Lambda\theta,
\]
where \(\ell(O;\theta)\) is the combined nuisance-model loss and
\(\Lambda\) is the ridge penalty matrix. We use the sign convention
\[
P U(O;\theta)
=
-\partial_\theta F_P(\theta),
\]
consistent with Section~\ref{sec:total-if}. Let
\(\widehat\theta^{\,I}\) minimize \(F_{P_{n,I}}\), and define the capped AIPW
estimate after nuisance refitting by
\[
\psihat_{M,\mathrm{refit}}^{\,I}
=
P_{n,I}
\{\varphi_M(O;\eta_{\widehat\theta^{\,I}})\}.
\]

The following condition controls the stability of nuisance refitting and
the resulting change in the capped AIPW score.

\begin{assumption}[Stability under feasible refitting]
\label{ass:refit-stability}
There exists a convex parameter region \(\Theta\) containing
\(\widehat\theta\) and every feasible refitted solution
\(\widehat\theta^{\,I}\). There also exist constants
\(\lambda>0\) and \(L_\varphi(M)<\infty\) such that, for every feasible
appended empirical distribution \(P_{n,I}\),
\(F_{P_{n,I}}\) is differentiable and \(\lambda\)-strongly convex on
\(\Theta\), with \(\widehat\theta^{\,I}\) its unconstrained minimizer.
Moreover,
\[
\left|
\varphi_M(z;\eta_{\theta_1})
-
\varphi_M(z;\eta_{\theta_2})
\right|
\le
L_\varphi(M)\|\theta_1-\theta_2\|_2
\]
for all \(z\) in the reference-sample or catalog support and all
\(\theta_1,\theta_2\in\Theta\).
\end{assumption}

Define
\[
B_U
=
\max_{z\in\Ccal}
\|U(z;\widehat\theta)\|_2.
\]
Under the finite-dimensional nuisance models considered here, a finite
catalog, bounded covariates and outcomes, a finite cap \(M\), and a
bounded parameter region make \(B_U\) and \(L_\varphi(M)\) finite. The next theorem establishes a finite-budget bound under nuisance refitting.

\begin{theorem}
\label{thm:refit-bound}
Under Assumption~\ref{ass:refit-stability}, every feasible set \(I\) of
\(k\) appended records satisfies
\begin{equation}\label{eq:refit-bound}
s\left\{
\psihat_{M,\mathrm{refit}}^{\,I}
-
\psihat_M
\right\}
\le
B^{\mathrm{fix}}_{s,M}(k)
+
\frac{k}{n+k}
\frac{L_\varphi(M)B_U}{\lambda},
\quad
s\in\{-1,+1\}.
\end{equation}
Consequently, for an append budget of at most \(m\),
\begin{equation}\label{eq:refit-envelope}
A^{\mathrm{safe}}_{s,M}(m)
=
\max_{0\le k\le m}
\left\{
B^{\mathrm{fix}}_{s,M}(k)
+
\frac{k}{n+k}
\frac{L_\varphi(M)B_U}{\lambda}
\right\}
\end{equation}
bounds the directional change in the fully refitted capped estimate.
\end{theorem}

Theorem~\ref{thm:refit-bound} provides a sufficient condition for a
decision to remain unchanged under all feasible append attacks up to the
declared budget. If crossing a decision threshold requires movement
\(D>0\) in direction \(s\), then
\[
A^{\mathrm{safe}}_{s,M}(m)<D
\]
ensures that no feasible attack of at most \(m\) records can cross that
threshold, subject to Assumption~\ref{ass:refit-stability}. Conversely,
\(A^{\mathrm{safe}}_{s,M}(m)\ge D\) means only that the bound cannot
certify stability; it does not establish the existence of an attack
achieving such movement. A proof is given in Appendix~\ref{sec:supp-proof-refit-bound}, and an
explicit Lipschitz constant for the working models is derived in
Appendix~\ref{sec:supp-lipschitz}.

The bound also provides an explicit criterion for choosing the cap. Before
examining the candidate rankings, specify a finite set of candidate caps
\(\mathcal M\), an append budget \(m\), a decision threshold \(c\), and,
if desired, a tolerance \(\tau\) for departure from the uncapped estimate.
For each \(M\in\mathcal M\), we recompute the capped estimate, its standard
error, the cap-specific rankings, and
\(A^{\mathrm{safe}}_{s,M}(m)\). For a point-estimate decision, define the
movement required to reach \(c\) in direction \(s\) by
\[
D_{s,M}
=
\max\{0,\,s(c-\psihat_M)\}.
\]
A candidate cap guarantees decision stability at budget \(m\) if
\[
A^{\mathrm{safe}}_{s,M}(m)<D_{s,M}
\]
for every decision-relevant direction \(s\). For a decision based on a
Wald endpoint, \(D_{s,M}\) is instead the distance from the relevant
capped interval endpoint to \(c\), with the reference-sample standard
error held fixed during the audit.

Among the caps satisfying the required directional inequalities and, when
imposed,
\[
|\psihat_M-\psihat|\le\tau,
\]
we select the largest \(M\). This applies the weakest capping that provides
the declared finite-budget protection while limiting departure from the
uncapped analysis. If no candidate cap satisfies these conditions, the
audit provides no certified protection at budget \(m\).

For practical candidate ranking under each cap, replace the uncapped score
in Section~\ref{sec:total-if} by \(\varphi_M\). Because capping modifies
only the final AIPW score, \(\widehat\theta\), \(H\), and \(U\) remain
unchanged, while
\[
\bar q_M
=
\Pn\partial_\theta
\varphi_M(O;\eta_{\widehat\theta})
\]
and
\[
\operatorname{IF}_{\mathrm{tot},M}(z)
=
\widehat r_M(z)
+
\bar q_M^{\!\top}H^{-1}U(z;\widehat\theta)
\]
depend on \(M\). Applying the same capacity-constrained scan defines
\(B^{\mathrm{tot}}_{s,M}(k)\) and \(A^{\mathrm{tot}}_{s,M}(m)\) as in
\eqref{eq:total-curve}--\eqref{eq:total-at-most}. Candidates are therefore re-ranked for every cap using
\(s\,\operatorname{IF}_{\mathrm{tot},M}(z_j)\). The reference-sample standard error of \(\psihat_M\) is estimated from the empirical variance of the centered capped scores
\(\widehat r_M(O_i)\); the capped total-influence values are used for
candidate ranking.

The finite-budget bound may be conservative, so we use total influence
for candidate ranking and the bound for certifying decision stability.
The guarantee is conditional on the fixed preprocessing, chosen cap,
specified catalog and capacities, and
Assumption~\ref{ass:refit-stability}; it excludes uncapped weights,
unstable nuisance learners, and attacks on preprocessing. When a Wald
endpoint is used, the resulting calculation is a decision-sensitivity
assessment rather than a post-selection coverage guarantee.

\section{Experimental design}\label{sec:design}

We use separate experiments to evaluate distinct components of the proposed auditing framework.
We first validate the method's two core calculations: exhaustive
enumeration confirms the exact fixed-pipeline curve, while one-record
appends across 20 synthetic samples evaluate the accuracy of total
influence (Section~\ref{sec:validation-results}). Larger synthetic studies examine
overlap, nuisance misspecification, attack structure, and inverse-weight
capping (Sections~\ref{sec:synthetic-results} and
\ref{sec:calibration-results}). A multisite simulation evaluates
hierarchical source capacities (Section~\ref{sec:governance-results}), and
two public-data analyses illustrate the audit beyond fully synthetic
settings (Section~\ref{sec:external-results}).

Continuous covariates are standardized using the reference sample, and an
intercept is included. Propensity and arm-specific outcome models use ridge
penalties of \(0.01\) in simulations and \(0.02\) in the external
analyses. Fitted propensities are clipped to
\([10^{-7},1-10^{-7}]\) only for numerical stability; inverse-weight caps
\(M\) are applied separately. Every realized attack refits both nuisance
models. The primary metric is realized directional movement after
refitting. Additional metrics include fixed-pipeline and total-influence
movement, repeated-record and random baselines, clean and attacked root
mean squared error, and one-record prediction accuracy. Random baselines
use distinct profiles and average three draws.

Except in the exhaustive checks, each observed covariate profile generates
four candidate records by crossing \(a\in\{0,1\}\) with the first and
ninety-ninth percentiles of the reference residuals in arm \(a\). Each
candidate outcome equals the fitted arm-specific mean plus the
corresponding residual percentile, clipped to the fixed outcome range. The
four alternatives for one profile have joint capacity one. These common
procedures are held fixed across the designs below, which vary the
data-generating setting and source-capacity structure. Further details on catalog construction, nuisance fitting, numerical
clipping, random baselines, and standard-error calculation are provided
in Appendix~\ref{sec:supp-implementation}.

\paragraph{Synthetic evaluation.}
To study overlap, nuisance misspecification, attack structure, and weight
capping, we draw \(n=1500\) observations with \(X\sim N(0,I_8)\) and true
effect \(\tau=0.5\). In the correctly specified and outcome-misspecified
regimes,
\[
e(X)=\expit\{\alpha(0.8X_1-0.4X_2+0.25X_3)\}.
\]
Under propensity misspecification, the final term is replaced by
\(0.55X_1X_2\), which is omitted from the fitted linear logit. The
correctly specified outcome mean is linear; the misspecified mean
additionally contains a sine, quadratic, and interaction term. Gaussian
noise with standard deviation \(0.9\) is truncated at three standard
deviations. We consider both nuisances correct, outcome misspecified, and
propensity misspecified. Across
\(\alpha\in\{0.75,1.5,2.5\}\), append budgets of
\(0.5\%,1\%,2\%,5\%\), and 60 replications per regime--overlap
combination, this yields 540 reference data sets. Cap comparisons use 50
replications per overlap level and
\(M\in\{\infty,50,20,10,5,3\}\) at a \(2\%\) budget.

\paragraph{Multisite source capacities.}
To isolate the effect of hierarchical contribution limits, we simulate 40
pooled networks with eight sites, 20 authenticated contributors per site,
and ten reference records per contributor (\(n=1600\)). Profiles are
nested within contributors, contributors within sites, and sites within a
network. Site-specific treatment intercepts and an observed
contributor-risk variable concentrate limited overlap in a subset of
sources while preserving correct nuisance specification. We compare four
policies: capacity one per profile; the profile rule plus a network cap of
32; the profile rule plus four records per site and a network cap of 32;
and a hierarchy of one record per contributor, four per site, and 32
network-wide. For each target policy and budget, we refit candidate sets
initialized from fixed-pipeline and total-influence rankings under all four
policies, retaining only sets feasible under the target policy. This
multi-start procedure prevents a policy from appearing safer because it
was evaluated using a weaker attack initialization. The complete policy definitions and multi-start implementation are given
in Appendix~\ref{sec:supp-governance}.

\paragraph{Public-data analyses.}
We apply the audit to one semi-synthetic benchmark and one observational
analysis. The Atlantic Causal Inference Conference (ACIC) 2016 benchmark combines
real covariates with simulated treatments and potential outcomes; we
analyze ten instances with \(n=4802\) and 79 encoded covariates
\citep{dorie2019}. The National Health and Nutrition Examination Survey
Epidemiologic Follow-up Study (NHEFS) data contain 1566 participants and
18 supplied covariates; treatment is smoking cessation and the outcome is
ten-year weight change \citep{hernan2020}. Both data sets are loaded using
the Python package \texttt{causallib} \citep{shimoni2019}. ACIC permits
evaluation against known potential outcomes, whereas the smoking analysis
is a stress test of the fitted pipeline rather than a new identification
analysis. Deterministic seeds, tests, outputs, and figure scripts are
provided with the manuscript.

\section{Results}\label{sec:results}

Following the experimental design above, we first validate the exact and
first-order calculations and then examine adversarial vulnerability across
overlap, nuisance specification, source capacities, inverse-weight caps,
and public-data applications. Together, these analyses show how the
different components of the proposed audit perform and what each
contributes in practice. Additional numerical summaries are provided in
Appendix~\ref{sec:supp-numerical}.

\subsection{Exact optimization and local refit prediction}
\label{sec:validation-results}

We first verify the exact fixed-pipeline calculation and tests
whether total influence predicts one-record movement after nuisance
refitting. Exhaustive enumeration on small grouped catalogs reproduced
the greedy nested-capacity solution exactly, and direct fixed-pipeline
recomputation matched Theorem~\ref{thm:exact} to machine precision.

Across the one-record checks from 20 synthetic reference samples, the
fixed-pipeline calculation had root mean squared error \(0.00563\) and
correlation \(0.990\) with realized refitted movement. Total influence
reduced the error to \(0.00395\), a \(29.9\%\) improvement, and increased
the correlation to \(0.999\) (Figure~\ref{fig:refit-prediction}). Its
relative error for the most adverse candidate in a held-out check was
\(5.8\%\). The refitting term also changed candidate ordering when
treatment residuals and propensity scores acted in opposing directions.
Thus, total influence is an accurate local ranking device, but not a
large-budget guarantee.

\begin{figure}[t]
\centering
\includegraphics[width=0.94\textwidth]{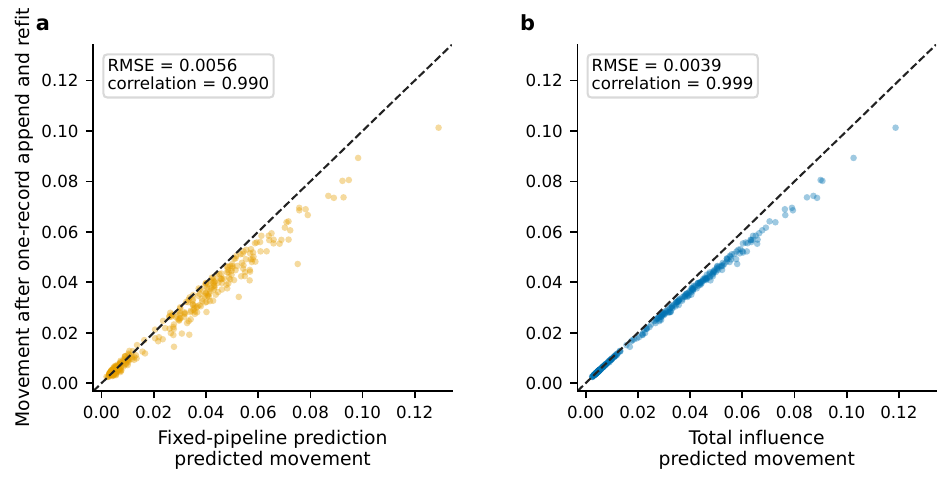}
\caption{One-record prediction of realized movement after refitting both
nuisance models.
\textbf{A.} Fixed-pipeline prediction.
\textbf{B.} Total-influence prediction.
Points pool candidate-level checks from 20 synthetic reference samples;
the dashed line is equality.}
\alttext{Two scatterplots compare predicted with realized movement after
one-record appends. Fixed-pipeline predictions show visible deviation from
the equality line, whereas total-influence predictions lie much closer to
it.}
\label{fig:refit-prediction}
\end{figure}

\subsection{Overlap, misspecification, and attack dispersion}
\label{sec:synthetic-results}

We next examine how overlap, nuisance misspecification, and attack
dispersion affect realized movement. Poor overlap sharply increased
vulnerability, and strategic appends remained effective when either
nuisance model was correctly specified. At a \(1\%\) budget, with both
nuisances correct, median downward movement rose from \(0.113\) under good
overlap to \(0.413\) under moderate overlap and \(1.761\) under poor
overlap, relative to a true effect of \(0.5\)
(Table~\ref{tab:synthetic-one}). Random plausible appends remained centered
within \(0.001\) of zero. Under moderate overlap, movement was \(0.578\)
with a misspecified outcome model and \(0.256\) with a misspecified
propensity model. Double robustness therefore protects consistency under
a union model but does not bound directional movement from feasible
appends.

\begin{table}[t]
\centering
\small
\caption{Median realized downward movement at a \(1\%\) append budget over
60 replications per cell. Capacity attacks select records by total
influence subject to profile capacities and then refit both nuisance
models; the final row reports random appends. The true effect is \(0.5\).}
\label{tab:synthetic-one}
\begin{tabular}{lrrr}
\toprule
Nuisance regime & \(\alpha=0.75\) & \(\alpha=1.50\) & \(\alpha=2.50\) \\
\midrule
Both correct & 0.113 & 0.413 & 1.761 \\
Outcome misspecified & 0.168 & 0.578 & 2.477 \\
Propensity misspecified & 0.097 & 0.256 & 0.736 \\
Random append, both correct & 0.000 & \(-0.001\) & 0.001 \\
\bottomrule
\end{tabular}
\end{table}

Attack dispersion became increasingly important as the budget grew. Under
moderate overlap, repeating the single worst record moved the estimate by
\(0.368\) at a \(0.5\%\) budget, compared with \(0.352\) for the
profile-capacity attack. At \(2\%\), the comparison reversed to \(0.277\)
versus \(0.452\), and at \(5\%\) to \(0.247\) versus \(0.507\).
Repeated records were increasingly absorbed by nuisance refitting,
whereas profile capacities spread the attack across distinct covariate
profiles. Exact fixed-pipeline movement increased from \(0.569\) to
\(2.248\) over the same budget range; its widening gap from realized
movement shows that it can be conservative after refitting
(Figure~\ref{fig:synthetic}).

\begin{figure}[t]
\centering
\includegraphics[width=\textwidth]{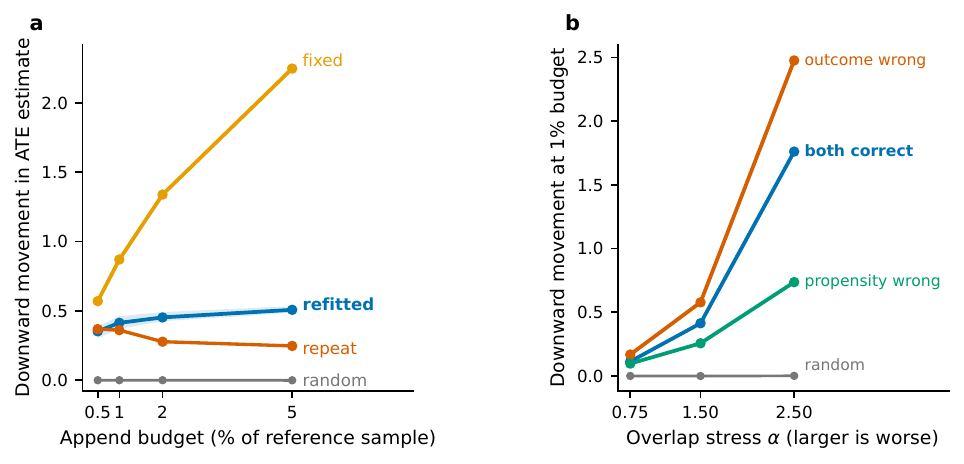}
\caption{Synthetic poisoning results.
\textbf{A.} Downward movement under moderate overlap with both nuisance
models correctly specified. The fixed-pipeline curve holds the nuisance
fits fixed; the refitted curve selects by total influence and then refits.
Shading shows the interquartile range.
\textbf{B.} Realized refitted movement at a \(1\%\) budget across overlap
and nuisance-specification regimes.}
\alttext{Two-panel line chart of synthetic attacks. Panel A compares
fixed-pipeline, refitted, repeated-record, and random movement as the
append budget increases. Panel B shows that realized movement at a
one-percent budget increases sharply as overlap worsens, including when
either nuisance model is correctly specified.}
\label{fig:synthetic}
\end{figure}

\subsection{Hierarchical capacities reduce attainable movement}
\label{sec:governance-results}

This analysis tests whether distributing contribution limits across
network, site, and contributor levels reduces attack concentration and
attainable movement. The reduction was substantial with the nuisance fits
fixed, but more modest after refitting. At a \(5\%\) budget, median exact
fixed-pipeline movement fell from \(4.47\) under profile-only capacity to
\(2.62\) after adding a network cap of 32. Distributing those slots as
four per site reduced movement to \(1.84\), and adding one record per
contributor reduced it to \(1.79\), \(60.0\%\) below profile-only
capacity. Policies with a network cap plateaued after all 32 slots were
used.

After nuisance refitting, median downward movement fell from \(0.605\)
under profile-only capacity to \(0.552\) under the full hierarchy, an
\(8.8\%\) reduction. The hierarchy nevertheless changed who could dominate
the attack: the most represented site's median share fell from \(36.9\%\)
to \(12.5\%\), and the maximum number of selected records from one
contributor fell from five to one. Hierarchical capacities therefore
strongly restrict direct score exposure and attack concentration, although
nuisance-mediated vulnerability remains
(Figure~\ref{fig:governance}).

\begin{figure}[t]
\centering
\includegraphics[width=\textwidth]{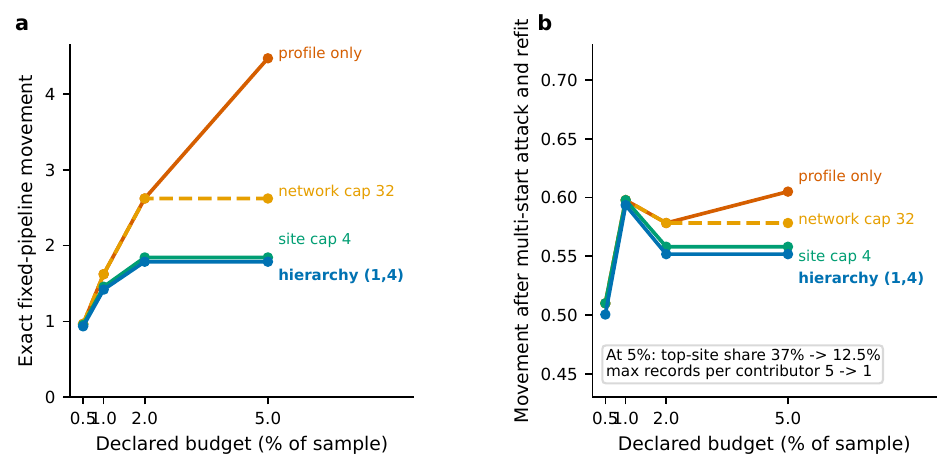}
\caption{Multilevel governance-capacity experiment over 40 pooled
networks; lines show medians.
\textbf{A.} Exact fixed-pipeline movement under four source-capacity
policies.
\textbf{B.} Realized movement after the multi-start attack and nuisance
refitting.}
\alttext{Two-panel line chart comparing four source-capacity policies.
Hierarchical site and contributor limits substantially reduce exact
fixed-pipeline movement and produce a smaller reduction in realized
movement after nuisance refitting.}
\label{fig:governance}
\end{figure}

\subsection{Weight caps reduce attacked error, but the bound is conservative}
\label{sec:calibration-results}

We then evaluate whether inverse-weight caps reduce attacked error,
whether the cap-specific local audit ranks their performance, and whether
the finite-budget bound is practically sharp. Under poor overlap, attacked
root mean squared error fell from \(1.369\) without capping to \(0.576\),
\(0.390\), \(0.266\), and \(0.225\) at
\(M=20,10,5,\) and \(3\), respectively. The \(83.6\%\) reduction at
\(M=3\) occurred without a clean-error penalty in this correctly specified
design: clean root mean squared error fell from \(0.0698\) to \(0.0618\).
This favorable clean-data result need not persist under outcome
misspecification or a different catalog.

The cap-specific total-influence summary ranked attacked performance well.
Across 18 overlap--cap settings, the Spearman correlation between
\[
\left\{
\widehat{\mathrm{se}}(\psihat_M)^2+
\bigl[A_{-1,M}^{\mathrm{tot}}(m)\bigr]^2
\right\}^{1/2}
\]
and realized root mean squared error under the \(2\%\) attack was \(0.95\);
each attack was re-optimized for its cap. The finite-budget envelope
exceeded all 750 corresponding realized movements, but its smallest median
envelope-to-realized ratio was approximately \(1.09\times10^4\). Thus,
total influence was informative for comparing caps, whereas the envelope
provided a conservative finite-budget upper bound rather than a practical
prediction or tuning criterion (Figure~\ref{fig:calibration}).

\begin{figure}[t]
\centering
\includegraphics[width=0.95\textwidth]{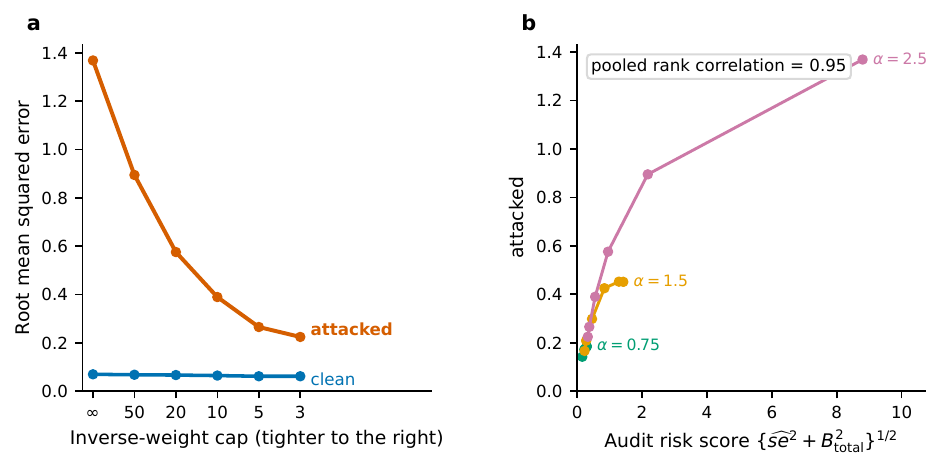}
\caption{Inverse-weight cap comparison at a \(2\%\) append budget.
\textbf{A.} Clean and attacked root mean squared error under poor overlap.
\textbf{B.} Cap-specific local audit summary versus realized attacked
error across 18 overlap--cap settings. Each attack is re-optimized for its
cap.}
\alttext{Two-panel comparison of inverse-weight caps. Panel A shows that
attacked root mean squared error decreases as the cap tightens while clean
error remains low. Panel B shows a strong positive association between the
cap-specific audit score and realized attacked error across overlap
settings.}
\label{fig:calibration}
\end{figure}

\subsection{Public-data audits show material sensitivity}
\label{sec:external-results}

Here we assess whether our findings persist beyond
the simulated data. In the ACIC benchmark, dispersed profile-limited
attacks remained more damaging than repeated copies of one record. Median
realized downward movement across ten instances was \(0.263\), \(0.415\),
and \(0.647\) at budgets of \(0.5\%\), \(1\%\), and \(2\%\),
respectively. At \(2\%\), the interquartile range was
\(0.510\)--\(0.802\), whereas unlimited repetition moved the estimate by
only \(0.242\), and random appends remained near zero.

The smoking-cessation analysis showed substantial sensitivity in both
directions. The uncapped reference estimate was \(3.40\) kg with a
reference-sample standard error of \(0.44\) kg. At a \(1\%\) budget,
corresponding to 16 records, the estimate moved \(1.34\) kg downward or
\(1.48\) kg upward, equivalent to \(3.09\) and \(3.39\) standard errors.
At \(2\%\), the corresponding movements were \(2.02\) and \(2.22\) kg.
At \(5\%\), the downward attack left an estimate of approximately
\(0.04\) kg without reversing its sign. A cap of five changed the
reference estimate only from \(3.40\) to \(3.42\) kg and reduced
\(1\%\) movement to \(1.12\)--\(1.21\) kg, but did not remove the
sensitivity. These findings show that small feasible append budgets can
materially change fitted causal conclusions; they do not imply that the
original data were manipulated (Figure~\ref{fig:external}).

\begin{figure}[t]
\centering
\includegraphics[width=\textwidth]{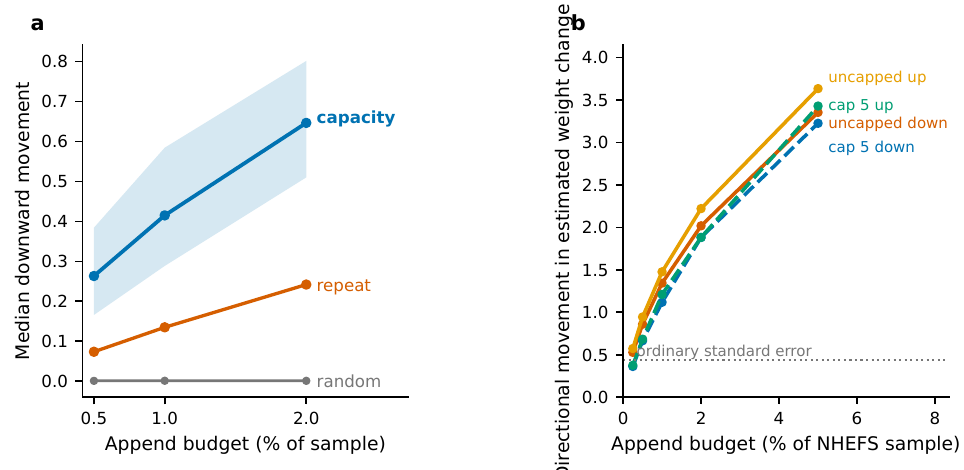}
\caption{Public-data audits.
\textbf{A.} ACIC 2016 median downward movement across ten instances;
shading shows the interquartile range.
\textbf{B.} Smoking-cessation movement in both directions under uncapped
and cap-five analyses; the dotted line is the uncapped reference-sample
standard error.}
\alttext{Two-panel public-data analysis. Panel A shows increasing downward
movement in ACIC as the append budget grows, with profile-limited attacks
exceeding repeated and random appends. Panel B shows increasing upward and
downward movement in the smoking-cessation estimate, with somewhat smaller
movement under a cap of five.}
\label{fig:external}
\end{figure}

\subsection{Scope and limitations}\label{sec:limitations}

The audit is conditional on the declared catalog and contribution
limits. Broader outcome ranges, looser covariate rules, or larger
capacities strengthen the adversary. Catalogs based on observed profiles
exclude novel covariate patterns and need not upper-bound a real attack,
while marginal residual quantiles may violate unmodeled conditional
restrictions. Domain-specific generators, contributor logs, and ingestion
policies would improve catalog realism.

The exact result applies to a fixed-pipeline AIPW estimator with nested
source groups; cross-classified constraints require more general
combinatorial optimization. Total influence is a local approximation, and
the finite-budget bound requires strongly convex finite-dimensional
nuisance models, bounded records, and capped inverse weights. We do not
establish poisoning-uniform coverage, a sharp global refit bound, or theory
for flexible cross-fitted learners, and decision-instability budgets hold
the reference-sample standard error fixed. Honest additions that change
the target population may represent estimand drift rather than poisoning.

The empirical evaluation also has limitations. The governance labels are
simulated, ACIC provides only ten instances, and the smoking-cessation
analysis has no ground-truth effect. Agreement across these settings
supports the proposed auditing framework but does not establish a
universal vulnerability rate.

\section{Conclusion}\label{sec:conclusion}

A causal estimate can satisfy familiar model-based diagnostics and still
be vulnerable to strategically selected records. We develop an append-only
data-poisoning audit that distinguishes three levels of analysis: exact
movement under a fixed pipeline, a first-order diagnostic under nuisance
refitting, and a conservative finite-budget bound for the fully refitted
capped estimator. This distinguishes exact fixed-pipeline results from local refit
approximations and condition-dependent finite-budget guarantees.

The experiments verify the exact greedy calculation, show that total
influence improves local refit prediction, and demonstrate that strategic
appends can materially move AIPW estimates even when only one nuisance model is correctly specified. Random additions poorly represented a direction-aware
adversary, and dispersed profile-limited attacks could be more damaging
than repeated copies of one extreme record. Source capacities and
inverse-weight caps reduced, but did not eliminate, this vulnerability.
The finite-budget bound exceeded all realized movements but was too
conservative for practical prediction or cap selection.

Current theory is limited to finite-dimensional penalized nuisance models,
fixed preprocessing, capped inverse weights, and nested source capacities;
total influence is local, and the finite-budget bound is conservative.
Future work should extend the audit to flexible cross-fitted learners,
cross-classified source constraints, and sharper global refit guarantees,
potentially using reachability methods such as AGT. For current AIPW
analyses, we recommend prespecifying plausible record catalogs and
contribution limits, reporting directional decision-instability budgets,
and distinguishing exact fixed-pipeline results from local refit
diagnostics and finite-budget bounds.

\section*{Data availability}

The Atlantic Causal Inference Conference 2016 and National Health and
Nutrition Examination Survey Epidemiologic Follow-up Study data are
publicly available and were loaded using the Python package
\texttt{causallib}.


\section*{Funding}

This work was supported by the Samsung Science and Technology Foundation (Project No. SSTF-BA2502-01).

\section*{Conflicts of interest}
None declared.

\bibliographystyle{plainnat}
\bibliography{references}

\newpage
\bigskip

\appendix
\input{supplement}

\end{document}

%% file: supplement.tex
\section*{\Large \centerline{SUPPLEMENTARY MATERIALS}}
\addcontentsline{toc}{section}{Supplementary Material}

\numberwithin{equation}{section}
\makeatletter
\@addtoreset{table}{section}
\@addtoreset{figure}{section}
\makeatother
\renewcommand{\thetable}{\thesection.\arabic{table}}
\renewcommand{\thefigure}{\thesection.\arabic{figure}}
\renewcommand{\theHequation}{\thesection.\arabic{equation}}
\renewcommand{\theHtable}{\thesection.\arabic{table}}
\renewcommand{\theHfigure}{\thesection.\arabic{figure}}

\section{Proofs of the Main Results}
\label{sec:supp-proofs}

This section proves the exact fixed-pipeline result, derives the
total-influence expression, and establishes the finite-budget bound under
nuisance refitting.

\subsection{Proof of Theorem~\ref{thm:exact}}
\label{sec:supp-proof-exact}

\begin{proof}
Fix \(s\in\{-1,+1\}\) and abbreviate
\[
d_j=d_j(s)
=
s\{\varphi(z_j;\etahat)-\psihat\}.
\]
For any feasible set \(I\) of cardinality \(k\), the definition of the
fixed-pipeline estimate gives
\begin{align}
s\{\psihat_{\mathrm{fix}}^{\,I}-\psihat\}
&=
s\left[
\frac{
n\psihat+\sum_{j\in I}\varphi(z_j;\etahat)
}{n+k}
-\psihat
\right]
\nonumber\\
&=
\frac{s}{n+k}
\left\{
n\psihat+\sum_{j\in I}\varphi(z_j;\etahat)
-(n+k)\psihat
\right\}
\nonumber\\
&=
\frac{s}{n+k}
\left\{
\sum_{j\in I}\varphi(z_j;\etahat)-k\psihat
\right\}
\nonumber\\
&=
\frac{1}{n+k}
\sum_{j\in I}
s\{\varphi(z_j;\etahat)-\psihat\}
\nonumber\\
&=
\frac{1}{n+k}\sum_{j\in I}d_j.
\label{eq:supp-fixed-movement}
\end{align}
Thus, for fixed \(k\), maximizing directional movement is equivalent to
maximizing \(\sum_{j\in I}d_j\) over feasible sets of cardinality \(k\).

Order the candidates in nonincreasing \(d_j\), using a fixed
deterministic rule to break ties, and accept each candidate whose
inclusion preserves all capacities. Let
\[
\pi_{s,1},\ldots,\pi_{s,R}
\]
denote the accepted order, where
\[
R
=
\max\left\{
|I|:
I\subseteq\{1,\ldots,J\},
\,
|I\cap L|\le c_L
\text{ for every }L\in\mathcal L
\right\}
\]
is the maximum feasible cardinality. For \(r=0,\ldots,R\), define
\[
G_r
=
\{\pi_{s,1},\ldots,\pi_{s,r}\},
\quad
G_0=\varnothing.
\]

Fix \(k\le R\). We show by induction that, for every
\(r=0,\ldots,k\), there exists a maximum-gain feasible set of cardinality
\(k\) containing \(G_r\). The claim is immediate for \(r=0\).

Suppose that a maximum-gain feasible set \(S\) of cardinality \(k\)
contains \(G_{r-1}\), and write
\[
e=\pi_{s,r}.
\]
If \(e\in S\), no modification is needed. Suppose instead that
\(e\notin S\).

First consider the case in which \(S\cup\{e\}\) is feasible. Because
\(r-1<k\), choose any
\[
j\in S\setminus G_{r-1}
\]
and define
\[
S'
=
S\setminus\{j\}\cup\{e\}.
\]
Then \(S'\) is feasible, has cardinality \(k\), and contains \(G_r\).

Now suppose that \(S\cup\{e\}\) is infeasible. Then there is at least one
group \(L\in\mathcal L\) containing \(e\) whose capacity is attained by
\(S\):
\[
|S\cap L|=c_L.
\]
Among such groups, let \(L_\ast\) be minimal under inclusion. Because
\(G_{r-1}\cup\{e\}\) is feasible,
\[
|G_{r-1}\cap L_\ast|+1
\le
c_{L_\ast}
=
|S\cap L_\ast|.
\]
Hence, there exists
\[
j\in
(S\cap L_\ast)\setminus G_{r-1}.
\]
Define
\[
S'
=
S\setminus\{j\}\cup\{e\}.
\]

We verify that \(S'\) is feasible. Consider any \(K\in\mathcal L\). If
\(e\notin K\), the number selected from \(K\) cannot increase. If both
\(e\) and \(j\) belong to \(K\), its selected count is unchanged.

It remains to consider the case \(e\in K\) and \(j\notin K\). Since
\(e\in K\cap L_\ast\), nestedness implies that either
\(K\subseteq L_\ast\) or \(L_\ast\subseteq K\). The latter is impossible
because \(j\in L_\ast\) but \(j\notin K\). Therefore,
\[
K\subsetneq L_\ast.
\]
By the minimality of \(L_\ast\), the capacity of \(K\) is not attained by
\(S\). Since the capacities are integer-valued,
\[
|S\cap K|\le c_K-1.
\]
Adding \(e\) therefore preserves the capacity of \(K\). Thus, \(S'\) is
feasible.

In either case, candidate \(j\) cannot have appeared before \(e\) in the
gain-ordered scan. If it had, the accepted set when \(j\) was considered
would have been a subset of \(G_{r-1}\). Adding \(j\) to that set would
have remained feasible because
\[
G_{r-1}\cup\{j\}\subseteq S
\]
and \(S\) is feasible. Candidate \(j\) would therefore have been
accepted, contradicting \(j\notin G_{r-1}\). Hence,
\[
d_e\ge d_j.
\]
It follows that \(S'\) has gain no smaller than that of \(S\), so \(S'\)
is also a maximum-gain feasible set and contains \(G_r\).

By induction, there exists a maximum-gain feasible set of cardinality
\(k\) containing \(G_k\). Since \(G_k=\{\pi_{s,1},\ldots,\pi_{s,k}\}\) is a maximum-gain feasible
set of cardinality \(k\), applying \eqref{eq:supp-fixed-movement} with
\(I=G_k\) gives
\[
B_s^{\mathrm{fix}}(k)
=
\frac{1}{n+k}
\sum_{r=1}^k d_{\pi_{s,r}}(s).
\]
Finally, maximizing over feasible cardinalities up to \(m\) gives
\[
A_s^{\mathrm{fix}}(m)
=
\max_{0\le k\le m}B_s^{\mathrm{fix}}(k).
\]
If \(m>R\), the upper limit is replaced by \(R\).
\end{proof}

\subsection{Proof of Proposition~\ref{prop:total-if}}
\label{sec:supp-proof-total-if}

\begin{proof}
For \(P_t=(1-t)\Pn+t\delta_z\), define
\[
\Gamma(t,\theta)
=
P_tU(O;\theta).
\]
By continuous differentiability,
\(\Gamma\) is continuously differentiable near
\((0,\widehat\theta)\), and
\[
\partial_\theta\Gamma(0,\widehat\theta)
=
\Pn\partial_\theta U(O;\widehat\theta)
=
-H
\]
is nonsingular. The implicit-function theorem therefore gives a
differentiable local solution \(\theta(t)\).

Since
\[
P_tU\{O;\theta(t)\}
=
(1-t)\Pn U\{O;\theta(t)\}
+
t\,U\{z;\theta(t)\},
\]
the product and chain rules give
\[
\begin{aligned}
0
&=
\left.
\frac{\mathrm d}{\mathrm dt}
P_tU\{O;\theta(t)\}
\right|_{t=0}\\
&=
U(z;\widehat\theta)
-
\Pn U(O;\widehat\theta)
+
\Pn\partial_\theta U(O;\widehat\theta)\dot\theta(0)\\
&=
(\delta_z-\Pn)U(O;\widehat\theta)
+
\Pn\partial_\theta U(O;\widehat\theta)\dot\theta(0).
\end{aligned}
\]

Because
\[
\Pn U(O;\widehat\theta)=0
\]
and
\[
\Pn\partial_\theta U(O;\widehat\theta)=-H,
\]
we obtain
\[
0
=
U(z;\widehat\theta)-H\dot\theta(0),
\]
and hence
\[
\dot\theta(0)
=
H^{-1}U(z;\widehat\theta).
\]

Now define the refitted AIPW functional along the path by
\[
g(t)=P_t\varphi_{\theta(t)}.
\]
Proceeding as above, differentiating at \(t=0\) gives
\[
\begin{aligned}
g'(0)
&=
(\delta_z-\Pn)\varphi_{\widehat\theta}
+
\left\{
\Pn\partial_\theta
\varphi_{\widehat\theta}(O)
\right\}^{\!\top}
\dot\theta(0)\\
&=
\varphi_{\widehat\theta}(z)
-
\Pn\varphi_{\widehat\theta}(O)
+
\bar q^{\!\top}H^{-1}U(z;\widehat\theta).
\end{aligned}
\]
Because
\[
\Pn\varphi_{\widehat\theta}(O)=\psihat
\]
and
\[
\widehat\phi(z)
=
\varphi_{\widehat\theta}(z)-\psihat,
\]
this becomes
\[
g'(0)
=
\widehat\phi(z)
+
\bar q^{\!\top}H^{-1}U(z;\widehat\theta)
=
\operatorname{IF}_{\mathrm{tot}}(z).
\]
\end{proof}

For a fixed reference sample, a first-order Taylor expansion at \(t=0\) gives
\begin{equation}
\label{eq:supp-local-expansion}
P_t\varphi_{\theta(t)}
-
\psihat
=
t\,\operatorname{IF}_{\mathrm{tot}}(z)
+
o(t)
\quad
\text{as }t\to0.
\end{equation}
Appending one copy of \(z\) corresponds to
\[
t=\frac{1}{n+1},
\]
so
\[
\frac{\operatorname{IF}_{\mathrm{tot}}(z)}{n+1}
\]
is the resulting first-order approximation. 

\subsection{Proof of Theorem~\ref{thm:refit-bound}}
\label{sec:supp-proof-refit-bound}

\begin{proof}
Fix a feasible set \(I\) of \(k\) appended records and write
\[
P_{n,I}
=
\frac{n\Pn+\sum_{j\in I}\delta_{z_j}}{n+k}.
\]
By Assumption~\ref{ass:refit-stability},
\(\widehat\theta\) and \(\widehat\theta^{\,I}\) belong to the same convex
region \(\Theta\), and \(F_{P_{n,I}}\) is differentiable and
\(\lambda\)-strongly convex on \(\Theta\). Its gradient is therefore
\(\lambda\)-strongly monotone:
\[
\left\langle
\nabla F_{P_{n,I}}(\theta_1)
-
\nabla F_{P_{n,I}}(\theta_2),
\theta_1-\theta_2
\right\rangle
\ge
\lambda\|\theta_1-\theta_2\|_2^2
\]
for all \(\theta_1,\theta_2\in\Theta\).

Because \(\widehat\theta^{\,I}\) is an unconstrained minimizer and
\(F_{P_{n,I}}\) is differentiable,
\[
\nabla F_{P_{n,I}}(\widehat\theta^{\,I})=0.
\]
Taking
\[
\theta_1=\widehat\theta,
\quad
\theta_2=\widehat\theta^{\,I},
\]
and applying the Cauchy--Schwarz inequality gives
\[
\lambda
\|\widehat\theta^{\,I}-\widehat\theta\|_2
\le
\left\|
\nabla F_{P_{n,I}}(\widehat\theta)
\right\|_2.
\]
Using the sign convention above,
\[
\left\|
\nabla F_{P_{n,I}}(\widehat\theta)
\right\|_2
=
\left\|
P_{n,I}U(O;\widehat\theta)
\right\|_2.
\]
Since
\[
\Pn U(O;\widehat\theta)=0,
\]
we have
\[
\begin{aligned}
P_{n,I}U(O;\widehat\theta)
&=
\frac{1}{n+k}
\sum_{j\in I}U(z_j;\widehat\theta),\\
\left\|
P_{n,I}U(O;\widehat\theta)
\right\|_2
&\le
\frac{k}{n+k}B_U.
\end{aligned}
\]
Consequently,
\begin{equation}
\label{eq:supp-parameter-stability}
\|\widehat\theta^{\,I}-\widehat\theta\|_2
\le
\frac{k}{n+k}\frac{B_U}{\lambda}.
\end{equation}

Next, add and subtract the capped score evaluated at the reference
nuisance fit:
\[
\begin{aligned}
\psihat_{M,\mathrm{refit}}^{\,I}-\psihat_M
={}&
\left[
P_{n,I}\{\varphi_M(O;\etahat)\}
-
\Pn\{\varphi_M(O;\etahat)\}
\right]\\
&+
P_{n,I}
\left[
\varphi_M(O;\eta_{\widehat\theta^{\,I}})
-
\varphi_M(O;\etahat)
\right].
\end{aligned}
\]

The first term satisfies
\[
\begin{aligned}
P_{n,I}\{\varphi_M(O;\etahat)\}
-
\Pn\{\varphi_M(O;\etahat)\}
&=
\frac{1}{n+k}
\sum_{j\in I}
\left\{
\varphi_M(z_j;\etahat)-\psihat_M
\right\}\\
&=
\frac{1}{n+k}
\sum_{j\in I}\widehat r_M(z_j).
\end{aligned}
\]
Therefore, its movement in direction \(s\) is bounded by the exact
fixed-pipeline optimum:
\[
\frac{1}{n+k}
\sum_{j\in I}s\,\widehat r_M(z_j)
\le
B_{s,M}^{\mathrm{fix}}(k).
\]

For the second term, the support of \(P_{n,I}\) is contained in the union
of the reference sample and the candidate catalog. The Lipschitz condition
in Assumption~\ref{ass:refit-stability} therefore gives
\[
\begin{aligned}
\left|
P_{n,I}
\left[
\varphi_M(O;\eta_{\widehat\theta^{\,I}})
-
\varphi_M(O;\etahat)
\right]
\right|
&\le
P_{n,I}
\left|
\varphi_M(O;\eta_{\widehat\theta^{\,I}})
-
\varphi_M(O;\etahat)
\right|\\
&\le
L_\varphi(M)
\|\widehat\theta^{\,I}-\widehat\theta\|_2\\
&\le
\frac{k}{n+k}
\frac{L_\varphi(M)B_U}{\lambda},
\end{aligned}
\]
where the final inequality follows from
\eqref{eq:supp-parameter-stability}.

Combining the two terms yields
\[
s\left\{
\psihat_{M,\mathrm{refit}}^{\,I}
-
\psihat_M
\right\}
\le
B_{s,M}^{\mathrm{fix}}(k)
+
\frac{k}{n+k}
\frac{L_\varphi(M)B_U}{\lambda}.
\]
Finally, maximizing the right-hand side over all feasible cardinalities
\(0\le k\le m\) gives
\[
A_{s,M}^{\mathrm{safe}}(m)
=
\max_{0\le k\le m}
\left\{
B_{s,M}^{\mathrm{fix}}(k)
+
\frac{k}{n+k}
\frac{L_\varphi(M)B_U}{\lambda}
\right\}.
\]
If \(m\) exceeds the maximum feasible cardinality, the upper limit is
replaced by that cardinality. The case \(k=0\) is immediate.
\end{proof}

\section{Additional Technical and Simulation Details}
\label{sec:supp-technical}

This section provides the multi-record first-order expansion, an explicit
Lipschitz constant for the capped score, numerical verification of the
theoretical calculations, and further implementation details.

\subsection{First-order approximation for several appended records}
\label{sec:supp-multiple-records}

Under the conditions of Proposition~\ref{prop:total-if}, let \(Q\) be a
probability measure supported on the candidate catalog, and consider
\[
P_t=(1-t)\Pn+tQ.
\]
Let \(\theta_Q(t)\) solve
\[
P_tU\{O;\theta_Q(t)\}=0,
\qquad
\theta_Q(0)=\widehat\theta.
\]
Applying the same product- and chain-rule argument as in
Section~\ref{sec:supp-proof-total-if} gives
\[
\left.
\frac{\mathrm d}{\mathrm dt}
P_t\varphi_{\theta_Q(t)}
\right|_{t=0}
=
Q\{\operatorname{IF}_{\mathrm{tot}}(O)\}.
\]

For a set \(I\) of \(k\ge1\) appended records, define
\[
Q_I
=
\frac{1}{k}\sum_{j\in I}\delta_{z_j},
\qquad
t_k
=
\frac{k}{n+k}.
\]
The empirical distribution after appending \(I\) is exactly
\[
P_{n,I}
=
(1-t_k)\Pn+t_kQ_I.
\]
Along the path associated with \(Q_I\),
\[
\left.
\frac{\mathrm d}{\mathrm dt}
P_t\varphi_{\theta_I(t)}
\right|_{t=0}
=
\frac{1}{k}
\sum_{j\in I}
\operatorname{IF}_{\mathrm{tot}}(z_j).
\]
Differentiability at \(t=0\) therefore gives
\[
P_t\varphi_{\theta_I(t)}
-
\psihat
=
\frac{t}{k}
\sum_{j\in I}
\operatorname{IF}_{\mathrm{tot}}(z_j)
+
o(t)
\qquad
\text{as }t\to0.
\]
Evaluating the leading term at \(t=t_k\) gives
\begin{equation}
\label{eq:supp-multiple-if}
\frac{1}{n+k}
\sum_{j\in I}
\operatorname{IF}_{\mathrm{tot}}(z_j).
\end{equation}
We use \eqref{eq:supp-multiple-if} as the first-order approximation to
the refitted movement produced by \(I\). It is a local approximation and
is not claimed to be exact when \(k/(n+k)\) is nonnegligible.

\subsection{An explicit Lipschitz constant for the capped score}
\label{sec:supp-lipschitz}

Let \(b(x)\) denote the common feature vector, including the intercept,
and suppose
\[
\|b(x)\|_2\le R_b
\]
over the reference-sample and catalog support. Let
\[
e_\beta(x)
=
\expit\{b(x)^\top\beta\},
\]
and let the outcome regressions be clipped linear predictors,
\[
\mu_{a,\gamma_a}(x)
=
\Pi_{[y_L,y_U]}
\{b(x)^\top\gamma_a\},
\]
where \(\Pi_{[y_L,y_U]}\) is projection onto
\([y_L,y_U]\). Assume that the observed and candidate outcomes also lie
in this interval, and define
\[
R_Y=y_U-y_L,
\quad
C_M=\max\{1,M-1\}.
\]

Under these conditions, one valid Lipschitz constant in
Assumption~\ref{ass:refit-stability} is
\begin{equation}
\label{eq:supp-lipschitz-constant}
L_\varphi(M)
=
R_b
\left\{
(M-1)^2R_Y^2+2C_M^2
\right\}^{1/2}.
\end{equation}

To verify \eqref{eq:supp-lipschitz-constant}, first note that
\[
e_\beta(x)^{-1}
=
1+\exp\{-b(x)^\top\beta\},
\quad
[1-e_\beta(x)]^{-1}
=
1+\exp\{b(x)^\top\beta\}.
\]
Where an inverse weight is below \(M\), the magnitude of its derivative
with respect to the linear predictor is at most \(M-1\); after the cap
binds, the capped weight is constant. Thus, each capped inverse weight is
globally Lipschitz in \(\beta\) with constant
\((M-1)R_b\).

Because \(Y\) and the clipped outcome predictions lie in
\([y_L,y_U]\), every residual has magnitude at most \(R_Y\). At points of differentiability, the norm of the propensity-block gradient is bounded by
\[
(M-1)R_bR_Y.
\]
For the treated outcome block, the coefficient multiplying
\(\partial_{\gamma_1}\mu_{1,\gamma_1}\) is
\(1-Aw_{1,M}(X)\), whose absolute value is at most \(C_M\). The analogous
coefficient for the control block is
\(-1+(1-A)w_{0,M}(X)\), also bounded by \(C_M\). Projection onto
\([y_L,y_U]\) is one-Lipschitz, so each outcome block contributes at most
\(C_MR_b\). Combining the three parameter blocks in Euclidean norm gives
\eqref{eq:supp-lipschitz-constant}. The bound remains valid at the
weight-cap and outcome-clipping kink points because the resulting
functions are globally Lipschitz.

For logistic and squared-error nuisance losses, the unpenalized Hessians
are positive semidefinite. Therefore, if the ridge matrix satisfies
\[
\Lambda\succeq\lambda I,
\]
the objective \(F_P\) is uniformly \(\lambda\)-strongly convex. If an
intercept or another coefficient is unpenalized, the ridge penalty alone
does not establish uniform strong convexity; a positive lower bound on
the empirical loss curvature in those directions must also be verified.

\subsection{Numerical verification of the theoretical calculations}
\label{sec:supp-validation}

For the exact fixed-pipeline result, the test suite enumerates all
feasible subsets of a five-candidate partition catalog and compares every
fixed-cardinality optimum with the greedy curve. It repeats this comparison
for a seven-candidate, three-level profile--contributor--site hierarchy and
verifies that a deliberately crossed hierarchy is rejected by the
nestedness check. All discrepancies are zero to numerical precision. A
separate test appends the first twelve accepted candidates and matches the
analytic fixed-pipeline curve to twelve decimal places.

For the local refit calculation, 337 one-record checks compare the
fixed-pipeline and total-influence predictions with complete nuisance
refitting. The fixed-pipeline prediction has root mean squared error
\(0.005628\), mean absolute error \(0.003736\), and correlation
\(0.99036\). Total influence reduces these values to \(0.003947\),
\(0.002624\), and \(0.99920\), respectively. In a held-out check using
the most adverse candidate, predicted movement is \(-0.017794\) and
realized movement is \(-0.016820\), giving a relative error of
\(5.79\%\).

Across 750 finite-cap audits, no realized movement exceeds
\(A_{s,M}^{\mathrm{safe}}(m)\). The bound is highly conservative: the
smallest median ratio of the envelope to realized movement is
\(10874.7\), occurring under poor overlap with \(M=3\), and the ratios
increase for looser caps. These calculations support using the envelope
as a finite-budget upper bound rather than as a prediction of realized
movement or a practical cap-selection score.

\subsection{Catalog construction and nuisance fitting}
\label{sec:supp-implementation}

The feature map is estimated once from the reference sample and includes
an intercept. Continuous covariates are standardized using the
reference-sample means and standard deviations, and the same transformation
is used for all candidate and attacked data sets. Logistic ridge regression
is fit by Newton iterations with backtracking, while the arm-specific ridge
outcome regressions are solved in closed form.

Fitted propensities are clipped to
\[
[10^{-7},1-10^{-7}]
\]
only for numerical stability. This numerical clipping is distinct from
the scientific inverse-weight cap \(M\). All outcome predictions used in the AIPW score, together with candidate
outcomes, are clipped to the fixed reference-sample outcome range.

For each observed covariate profile, the catalog contains four candidate
records: a control record with a lower control-arm residual, a control
record with an upper control-arm residual, a treated record with a lower
treated-arm residual, and a treated record with an upper treated-arm
residual. The lower and upper residuals are the first and ninety-ninth
percentiles of the corresponding reference-arm residual distribution.
The four alternatives for one profile have joint capacity one. Because
they have identical contributor, site, and network memberships, they may
be reduced, separately for each direction, to their highest-gain
alternative before the global capacity scan.

The repeated-record comparator appends unrestricted copies of the
candidate with the largest fixed-pipeline gain. Random plausible appends
use distinct observed profiles, draw treatment from the reference marginal
treatment probability, and draw a residual from the central \(98\%\) of
the empirical residual distribution in the selected treatment arm.
Synthetic random baselines average three draws for each reference sample
and budget.

Every realized attack refits both nuisance models. The fixed-pipeline and
total-influence curves are audit quantities, whereas the primary empirical
metric is realized directional movement after nuisance refitting.

For a capped estimate, the reference-sample standard error used in the
implementation is
\begin{equation}
\label{eq:supp-capped-se}
\widehat{\mathrm{se}}(\psihat_M)
=
\left\{
\frac{1}{n(n-1)}
\sum_{i=1}^n
\widehat r_M(O_i)^2
\right\}^{1/2}.
\end{equation}
Thus, the standard error is based on the empirical variance of the
centered capped scores. Capped total-influence values are used for
refit-aware candidate ranking, not for the standard-error calculation.

\subsection{Governance-capacity simulation}
\label{sec:supp-governance}

The governance experiment contains eight sites, 20 authenticated
contributors per site, and ten reference records per contributor. Each
candidate is assigned to a profile, contributor, site, and network root.
The four policies are:

\begin{enumerate}
\item capacity one per profile;

\item capacity one per profile and capacity 32 for the network;

\item capacity one per profile, capacity four per site, and capacity 32
for the network; and

\item capacity one per profile, capacity one per contributor, capacity
four per site, and capacity 32 for the network.
\end{enumerate}

All four policies use nested source groups. At each budget, candidate sets
are generated from both fixed-pipeline and total-influence rankings under
every policy. For a given target policy, the implementation refits every
generated set that is feasible under the target capacities and retains the
largest downward movement. This multi-start construction prevents a less
restrictive policy from appearing safer merely because it was evaluated
with a weaker attack initialization.

Random seeds are indexed separately by scenario, overlap level,
replication, and baseline draw. The command-line runner supports
component-wise execution through \texttt{--section}, allowing simulations,
cap comparisons, public-data analyses, and summaries to be reproduced
independently.

\section{Additional Numerical Results}
\label{sec:supp-numerical}

\begin{table}[htbp]
\centering
\scriptsize
\setlength{\tabcolsep}{3.5pt}
\caption{Governance-capacity experiment at a declared \(5\%\) budget;
medians over 40 simulated data networks. Fixed-pipeline movement is exact
under the specified hierarchy. Realized movement is the strongest
feasible multi-start refit stress test.}
\label{tab:supp-governance}
\begin{tabularx}{\textwidth}{>{\raggedright\arraybackslash}Xrrrr}
\toprule
Policy
& Fixed-pipeline
& Realized
& Top-site share
& Max.\ per contributor\\
\midrule
Profile only
& 4.472 & 0.605 & 0.369 & 5\\
Network cap 32
& 2.622 & 0.578 & 0.203 & 3\\
Four per site, cap 32
& 1.844 & 0.558 & 0.125 & 3\\
One per contributor, four per site, cap 32
& 1.788 & 0.552 & 0.125 & 1\\
\bottomrule
\end{tabularx}
\end{table}

\begin{table}[htbp]
\centering
\small
\caption{Selected inverse-weight-cap results at a \(2\%\) budget. Root
mean squared errors use 50 replications per overlap level.}
\label{tab:supp-caps}
\begin{tabular}{rrrrr}
\toprule
\(\alpha\)
& Cap
& Clean error
& Attacked error
& Median local movement\\
\midrule
0.75 & \(\infty\) & 0.0527 & 0.1844 & 0.2870\\
0.75 & 3          & 0.0520 & 0.1421 & 0.1515\\
1.50 & \(\infty\) & 0.0593 & 0.4513 & 1.4210\\
1.50 & 3          & 0.0581 & 0.1663 & 0.2251\\
2.50 & \(\infty\) & 0.0698 & 1.3689 & 8.7971\\
2.50 & 3          & 0.0618 & 0.2245 & 0.3258\\
\bottomrule
\end{tabular}
\end{table}

\begin{table}[htbp]
\centering
\small
\caption{Realized refit-aware movement in the public-data analyses. ACIC
values are medians across ten instances; smoking-cessation values are
directional movements from one reference analysis.}
\label{tab:supp-external}
\begin{tabular}{lrrr}
\toprule
Analysis and direction
& \(0.5\%\)
& \(1\%\)
& \(2\%\)\\
\midrule
ACIC benchmark, downward
& 0.263 & 0.415 & 0.647\\
Smoking cessation, downward
& 0.859 & 1.344 & 2.019\\
Smoking cessation, upward
& 0.945 & 1.477 & 2.222\\
Smoking cessation, cap five, downward
& 0.665 & 1.118 & 1.882\\
Smoking cessation, cap five, upward
& 0.684 & 1.211 & 1.883\\
\bottomrule
\end{tabular}
\end{table}